\documentclass[letterpaper]{article} % DO NOT CHANGE THIS
\usepackage{text21}  % DO NOT CHANGE THIS
\usepackage{times}  % DO NOT CHANGE THIS
\usepackage{helvet} % DO NOT CHANGE THIS
\usepackage{courier}  % DO NOT CHANGE THIS
\usepackage[hyphens]{url}  % DO NOT CHANGE THIS
\usepackage{graphicx} % DO NOT CHANGE THIS
\usepackage{natbib}  % DO NOT CHANGE THIS AND DO NOT ADD ANY OPTIONS TO IT
\usepackage{caption} % DO NOT CHANGE THIS AND DO NOT ADD ANY OPTIONS TO IT
\usepackage{graphicx}
\usepackage{amsmath,amssymb}
\usepackage{mathtools}
\usepackage{dsfont}
\usepackage{xspace}
\usepackage{epsfig}
\usepackage{color}
\usepackage{url}
\usepackage{tikz}
\usepackage{float}
\usepackage{booktabs}
\usepackage{makecell}
\usepackage{multicol}
\usepackage{multirow}
\usepackage{longtable}
\usepackage{array}
\usepackage{anyfontsize}
\usepackage{balance}
\usepackage{algpseudocode}
\newcommand{\ignore}[1]{}
\usepackage[vlined,linesnumbered,ruled,titlenotnumbered]
{algorithm2e}
\makeatletter
\providecommand{\@LN}[2]{}
\makeatother

\title{Computing Diverse Sets of Solutions for Monotone Submodular Optimisation Problems}

 \author {
     Aneta Neumann, \textsuperscript{\rm 1}
       Jakob Bossek, \textsuperscript{\rm 2}
     Frank Neumann \textsuperscript{\rm 1} \\
 }
\affiliations {
     \textsuperscript{\rm 1} School of Computer Science, The University of Adelaide, Australia \\
     \textsuperscript{\rm 2} Statistics and Optimization, The University of M{\"u}nster, Germany \\
 }

\usepackage{amsmath,amssymb,amsthm}

\newtheorem{theorem}{Theorem}
\usepackage{graphicx}
\usepackage{mathtools}
\usepackage{xcolor}
\newcommand{\R}{\mathbb{R}}
\newcommand{\argmax}{\text{arg$\,$max}}

\begin{document}

\maketitle
\begin{abstract}
Submodular functions allow to model many real-world optimisation problems. 
This paper introduces approaches for computing diverse sets of high quality solutions for submodular optimisation problems. We first present diversifying greedy sampling approaches and analyse them with respect to the diversity measured by entropy and the approximation quality of the obtained solutions. Afterwards, we introduce an evolutionary diversity optimisation approach to further improve diversity of the set of solutions. We carry out experimental investigations on popular submodular benchmark functions that show that the combined approaches achieve high quality solutions of large diversity.
\end{abstract}

\section{Introduction}

Submodular functions play a key role in the area of optimisation as they allow to model many real-world optimisation problem~\cite{DBLP:books/cu/p/0001G14}. The classical goal in the area of optimisation is to compute one single solution that maximises a given objective function under a given set of constraints. However, in many application areas it is useful to compute a set of diverse and high quality solutions that can be used for further discussion on which solution to implement. Having a set of high quality solutions available also allows to see the different alternatives and gives a broader perspective on options for implementing a solution. 

\subsection{Related work}
The classical problem of maximizing a monotone submodular function under a given uniform constraint can be solved by a simple greedy algorithm given in~\cite{DBLP:journals/mp/NemhauserWF78}.
Extensions have been made in terms of knapsack and more general cost constraints~\cite{DBLP:conf/aaai/ZhangV16} as well as results for function being close to submodular have been obtained for generalized greedy approaches that pick in each step an element with a largest marginal gain. 
In recent years, it has been shown for a variety of the beforehand mentioned problems that evolutionary multi-objective algorithms using a Pareto optimisation approach~\cite{DBLP:journals/ec/FriedrichN15,DBLP:conf/nips/QianYZ15} can achieve the same performance guarantees as greedy approaches, but perform much better in practice in a wide range of settings~\cite{DBLP:conf/ijcai/QianSYT17,DBLP:conf/ppsn/NeumannN20}. An overview on Pareto optimisation approaches for submodular optimisation can be found in the recent book by~\citet{DBLP:books/sp/ZhouYQ19}.

In many real-world applications it is desirable to have a set of high quality solutions available that differ with respect to their design. This allows decision makers to see different options not quantified as part of the formulated objectives and see a wide range of option of achieving a desired goal. Recent investigations into this direction have gained some attention in the artificial intelligence community with respect to planning problems~\cite{DBLP:conf/aaai/0001S20,DBLP:conf/aaai/KatzSU20}. In the case of top-$k$ planning, the goal is to obtain a set of solutions such that no better solution outside that set exists~\cite{DBLP:conf/aips/0001SUW18}. Solutions for this problem are obtained by iteratively forbidding solutions or solution components. Furthermore, the design of diverse sets of policies has been investigated~\cite{DBLP:conf/ijcai/MasoodD19} and gradient-based methods have developed to diversify existing approaches for the creation of policies.

Computing diverse sets of high quality solutions for a given problem has recently gained increasing interest in the area of evolutionary computation under the notions evolutionary diversity optimisation~\cite{DBLP:conf/gecco/UlrichT11,DBLP:conf/ppsn/GaoNN16,DBLP:conf/gecco/NeumannGDN018,DBLP:conf/gecco/NeumannG0019} and quality diversity~\cite{DBLP:journals/firai/PughSS16,DBLP:conf/ppsn/PughSS16}. In terms of combinatorial optimisation problems, the computation of high quality diverse sets of solutions for the Traveling Salesperson Problem has been considered~\cite{DBLP:conf/gecco/DoBN020}.

\subsection{Our contribution}
We examine approaches for creating diverse sets of high quality solutions for submodular optimisation problems. The class of objective functions that we study are monotone and $\alpha_f \in [0,1]$-submodular where $\alpha_f$ measures how close a function $f$ is to being submodular.
We are interested in sets of solutions where each solution fulfills a given approximation criterion and aim to maximise the diversity of the set of solutions obtained. 

We present approaches for computing diverse sets of solutions for monotone functions under given constraints that all are of high quality. 
We first present diversifying greedy sampling approaches for monotone functions with a uniform and a knapsack constraint. Our approaches make use of established greedy approaches, but allow to create diverse solutions while still almost maintaining the approximation quality achieved by the greedy approaches. We show that solutions constructed by our diversifying greedy sampling approaches only loose a small amount in terms of approximation ratio compared to standard greedy approaches but are able to construct many different solutions.

Afterwards, we introduce an evolutionary algorithm approach to maximise diversity under the condition that all solutions fulfill the given quality criterion. We consider the classical entropy measure to measure the diversity of a given set of solutions.
Our combined approach for creating high quality diverse sets of solutions, first runs one of the developed diversifying greedy sampling approaches and uses evolutionary diversity optimisation to increase the diversity of the obtained set while maintaining the guaranteed approximation quality of the diversifying greedy sampling approaches

We show in our experimental investigations that the evolutionary diversity optimisation approach allows to significantly improve the diversity of the set of solutions obtained by diversifying greedy sampling for the classical submodular coverage problem. Furthermore, our experimental results reveal the trade-offs with respect to the diversity of the obtained sets and the quality guaranteed by the diversifying greedy sampling approaches.

The paper is structured as follows. In the next section, we introduce the problem of computing sets of high quality solutions. Afterwards we present our diversifying greedy sampling approaches when working with uniform constraints and knapsack constraints. We examine the short-comings of these approaches and introduce and evolutionary diversity optimisation approach to increase the diversity of our solutions sets and report on experimental results afterwards. Finally, we finish with some concluding remarks.

\section{Problem Formulation}

Given a set $V=\{v_1, \ldots, v_n\}$, an objective function $f \colon 2^V \rightarrow \mathds{R}_+$, and a constraint $c(X) \leq B$, the classical goal in optimisation is to find a solution $OPT = \arg \max_{X \subseteq V} \{f(X) \mid 
%X \subseteq V \wedge 
c(X) \leq B\}$.
Assume that a have two sets $A$ and $B$ such that $A \subseteq B \subseteq V$. The function $f$ is monotone iff we have $f(A) \leq f(B)$. A function $f$ is submodular if for  $v \notin B$, $f(A \cup \{v\})-f(A) \geq f(B \cup \{v\}) - f(B)$ holds.

The submodularity ratio $\alpha_f$~\cite{DBLP:conf/aaai/ZhangV16,DBLP:conf/ijcai/QianSYT17} of a given function $f$ is defined as

$$\alpha_f = \min_{A\subseteq B,v\not \in B}\frac{f(A\cup v) - f(A)}{f(B\cup v)- f(B)}\text{.}$$

We call a function $f$, $\alpha_f$-submodular in this case.
Note, that $f$ is submodular iff $\alpha_f=1$ holds. Throughout this paper, we assume that $f$ is is monotone and analyze approximation guarantees in terms of $\alpha_f$.

Given an objective function $f$ and a constraint specified by a cost function $c$ and constraint bound $B$, we study the problem of computing a (multi-) set of solutions $P=\{P_1, \ldots, P_{\mu}\}$ such that $f(P_i) \geq f_{\min} \wedge c(P_i) \leq B$ holds, $1 \leq i \leq \mu$. 
Here $f_{\min}$ is a threshold value which requires a minimum objective value for any solution in $P$.

In the following, we introduce greedy sampling approaches that are able to construct sets of solutions where all solutions have guaranteed quality. We also point out the trade-off in terms of guaranteed quality of the solutions and the number of solutions that are meeting the guaranteed quality and present and evolutionary diversity optimisation approach that can significantly increase the entropy of the solutions sets.

\section{Diverse Sets for Uniform Constraints}

In the following, we show how to adapt popular greedy algorithms in the area of submodular optimisation such that they are able to construct diverse sets of high quality solutions.
We first consider the case of a uniform constraint, i.e. we have $c(X) = |X|\leq B$. 

\subsection{Diversifying Greedy Sampling}
The diversifying greedy sampling (DGS) approach outlined in Algorithm~\ref{alg:DGS} starts with an empty set and always includes the item with the largest marginal gain until $B-m$ elements have been inserted. The margin $m$ here determines the quality of the partial solution $S$ obtained in this way. Throughout this paper, we assume that $m$ is small compared to $B$, i.e. $m = o(B)$ holds and it should be noted that small values of $m$ are sufficient to construct large sets of diverse solutions if $n-B$ is sufficiently large. 
After having obtained $S$ a set of $\mu$ solutions $\{P_1, \ldots, P_{\mu}\}$ is constructed extending $S$ by adding $m$ randomly chosen elements that are not in $S$ in order to obtain $P_i$, $1 \leq i \leq \mu$.

We now analyze DGS with respect to the quality and the diversity of the solution set it obtains.
We investigate how to create sets of distinct high quality approximations for monotone submodular functions with linear constraints.
We start by investigating the case of uniform constraint.

\begin{theorem}
\label{thm:dgs}
For a given monotone $\alpha_f$-submodular function $f$ with a uniform constraint $|X|\leq B$ and margin $m\leq B$, DGS computes a population $P=\{P_1, \ldots, P_{\mu}\}$, where 
$f(P_i) \geq (1- e^{-\alpha_f}) \cdot (1- \alpha_f /B)^{-m})\cdot f(OPT)$,
$1 \leq i \leq \mu$.
If $m=o(B)$ and $\alpha_f$ is a constant, then for such solution $P_i$
$$f(P_i) \geq (1-o(1)) \cdot (1-e^{-\alpha_f}) \cdot f(OPT)$$
holds.
\end{theorem}

\begin{proof}
We consider a solution produced by the greedy algorithm which introduced $B-m$ elements. For the proof, we follow the ideas for the classical greedy approach analysed in~\cite{DBLP:journals/mp/NemhauserWF78} and take into account the submodularity ratio $\alpha_f$ and the fact that only $B-m$ elements can be introduced to obtain the solution $S$.

For each set $X$ with $|X| < B-m$ and the element $v^* \not \in X$ such that $f(X \cup v^*) - f(x)$ is maximal, we have 
\begin{align*}
f(OPT)
& \leq f(OPT \cup X) \\
& \leq f(X) + (B/\alpha_f) f(X \cup v^*) - f(X).
\end{align*}
Here the first inequality follows from monotonicity and the second one holds as $f$ is $\alpha_f$-submodular.

This implies that in each step an element $v^*$ is added to the current partial solution $S$, the function value increases by at least
$$f(S \cup v^*) \geq \frac{\alpha_f}{B} \cdot (f(OPT)-f(X)).
$$

Using induction on the number of elements $j$~ as done in~\cite{DBLP:journals/mp/NemhauserWF78} and taking into account $\alpha_f$, we obtain a partial solution $S_i$ with

$$f(S_i) \geq (1- (1-\alpha_f/B)^{j}) \cdot f(OPT)$$

After having inserted $B-m$ elements greedily to obtain $S$, we have 
\begin{eqnarray*}
f(S) &\geq & (1- (1-\alpha_f/B)^{B-m}) \cdot f(OPT)\\
&\geq & (1- e^{-\alpha_f}) \cdot (1-\alpha_f/B)^{-m}) \cdot f(OPT)
\end{eqnarray*}

If $m = o(B)$ and $\alpha_f$ a constant, then we have
\begin{eqnarray*}
f(S) &\geq & (1-o(1)) (1- e^{-\alpha_f})  \cdot f(OPT).
\end{eqnarray*}

Each solution $P_i$ is obtained from $S$ by adding $m$ additional elements. As $f$ is monotone, we have $f(P_i) \geq f(S)$ which completes the proof.
\end{proof}

The margin $m$ allows to select $m$ additional elements not contained in $S$ to obtain solutions meeting the quality criteria. This leads to a large number of solutions if the difference $n-B$ is sufficiently large.

\begin{theorem}
\label{thm:dgs-size}
There are at least $\sum_{i=0}^m \binom{n-B+m}{i}$ feasible solutions $X$ with $f(X) \geq (1- e^{-\alpha_f}) \cdot (1- \alpha_f /B)^{-m})\cdot f(OPT)$.
Each $P_i$ produced by DGS is sampled from a set of
$\binom{n-B+m}{m}$ such solutions.
\end{theorem}
\begin{proof}
As $B-m$ elements have only been introduced to obtain $S$, $m$ additional elements among the $n-B+m$ elements not chosen by the greedy algorithm can be introduced. Adding every possible subset consisting of $i$, $0 \leq i \leq m$, so far unchosen elements to the $B-m$ elements of $S$ gives a set of
$$\sum_{i=0}^m \binom{n-B+m}{i}$$
solutions. Each $P_i$ selects $m$ of such elements which implies that each $P_i$ is sampled from a set of 
$\binom{n-B+m}{m}$
different solutions. 
\end{proof}

Based on the previous proof, we have that the number of different solutions with exactly $B$ elements that meet the quality criterion of Theorem~\ref{thm:dgs} is at least $\binom{n-B+m}{m}$.
We will concentrate on such solutions with exactly $B$ elements as we are dealing with monotone functions where the addition of elements does not reduce the function value. In fact, adding elements without violating the constraint often leads to a solution of higher quality in practice.

\section{Diverse Sets for Knapsack Constraints}

We now consider the case where we have given a knapsack constraint
$c(S) \leq B$, where $c(S)=\sum_{s \in S} c(s)$ and $c \colon V \rightarrow \mathds{R}_+$ gives the cost of choosing an element $s$.
Let $c_{\min}= \min_{v \in V} c(v)$ and $c_{\max}= \max_{v \in V} c(v)$ be the minimal and maximal cost of the elements in $V$, respectively.

\subsection{Generalized Diversifying Greedy Sampling}
The generalized diversifying greedy sampling (GDGS) approach also works with a margin $m$ which reduces the threshold for introducing elements in the greedy steps from $B$ to $B-m$. In each step, the element with the largest gain to weight ratio without violating the constraint is chosen.

\begin{algorithm}[t]
	\SetKwInOut{Input}{input}
    \Input{
    Set of elements $V$, function $f$, budget constraint $B$, margin $m$, number of solutions $\mu$.}
    $S \leftarrow\emptyset$\;
		$V^\prime \leftarrow V$\;
		\While{$(|S|<B-m) \wedge (V' \not= \emptyset)$}{
$v^*\leftarrow \argmax_{v\in V^\prime} (f(S \cup\{v\})-f(S))\label{lineGA:KostenNutzen}$\;
    $S\leftarrow S\cup \{v^*\}$\;
    $V^\prime \leftarrow V^\prime\setminus \{v^*\}$\;
    }
    \For{$i=1, \ldots, \mu$}{
    $P_i \leftarrow S$\;
   $V^\prime \leftarrow V \setminus S$\;
   \While{$(|P_i|<B) \wedge (V' \not= \emptyset)$}{
    Choose $v^* \in V^\prime$ uniformly at random\;
    $V^\prime \leftarrow V^\prime\setminus \{v^*\}$\;
     }
     }
    \Return{$P=\{P_1, \ldots, P_{\mu}\}$}\;
    \caption{Diversifying Greedy Sampling (DGS)}\label{alg:GA}
    \end{algorithm}

    \begin{algorithm}[t]
	\SetKwInOut{Input}{input}
    \Input{
    Set of elements $V$, functions $f$ and $c$, budget constraint $B$, margin $m$, number of solutions $\mu$.}
    $S \leftarrow\emptyset$\;
		$V^\prime \leftarrow V$\;
\While{$V^\prime \not= \emptyset$}{$v^*\leftarrow \argmax_{v\in V^\prime}\frac{f(S \cup\{v\})-f(S)}{W(S\cup \{v\})-W(S)\label{line:KostenNutzen}}$\;
    \If{$W(S\cup \{v^*\})\leq B-m$} {$S\leftarrow S\cup \{v^*\}$\;
    }
   
    $V^\prime \leftarrow V^\prime\setminus \{v^*\}$\;}
    $T \leftarrow \argmax_{Y\in \{S,\{y\}\}}f(Y)$\;
     \For{$i=1, \ldots, \mu$}{
    $P_i \leftarrow T$\;
   $V^\prime \leftarrow V \setminus T$\;
    \While{$V^\prime \not= \emptyset$}{
    Choose $v^* \in V^\prime$ uniformly at random\;
    \If{$W(P_i\cup \{v^*\})\leq B$}{$P_i\leftarrow S\cup \{v^*\}$\;}
    $V^\prime \leftarrow V^\prime\setminus \{v^*\}$\;}
     }
    
    \Return{$P=\{P_1, \ldots, P_{\mu}\}$}\;
    \caption{Generalized Diversifying Greedy Sampling (GDGS)}\label{alg:DGS}
    \end{algorithm}

\begin{theorem}
Consider a monotone $\alpha_f$-submodular function $f$ with a knapsack constraint $c(X)\leq B$ and margin $m\leq B$.
If $\alpha_f \cdot m=o(B)$, then for each solution $P_i$, $1 \leq i \leq \mu$, constructed by GDGS, we have 
$$
f(P_i) \geq (\alpha_f/2)(1- e^{-\alpha_f+o(1)}) \cdot f(OPT)
$$
\end{theorem}

\begin{proof}
We run the generalized greedy algorithm for budget $B-m$. 

Following~\cite{DBLP:conf/ijcai/QianSYT17}, the gain in terms of $f$ when inserting $v^*$ is at least
$$f(S \cup v^*) - f(S) \geq \alpha_f \cdot \frac{c(S \cup v^*)-c(S)}{B} \cdot (f(OPT)-f(S)).
$$

GDGS obtains a solution $S$ with $|S|=L$ and an element $y$ such that
$c(S) + c(y) \geq B-m$. Using induction, we get

\begin{align*}
&f(S \cup y)  \geq  \left[1 -\prod_{j=1}^{L+1} \left(1- \alpha_f \frac{c(s)}{B}\right) \right] \cdot f(OPT)\\
& \geq  \left[1 -\prod_{j=1}^{L+1} \left(1- \alpha_f \frac{B-m}{(L+1)B}\right) \right] \cdot f(OPT)\\
& \geq  \left[1 - \left(1- \alpha_f \frac{1}{L+1} +\alpha_f \frac{m}{(L+1)\cdot B}\right)^{L+1} \right] \cdot f(OPT)
\end{align*}

Assuming $\alpha_f \cdot m = o(B)$, we have 
\begin{align*}
f(S \cup y) & \geq \left[1 - \left(1- \frac{\alpha_f-o(1)}{L+1}\right)^{L+1} \right]  f(OPT)\\
& =(1- e^{-\alpha_f+o(1)}) \cdot f(OPT)\\
\end{align*}

Using that $f$ is $\alpha_f$-submodular, we have 
\begin{eqnarray*}
f(S) + f(y) & \geq & f(S) + \alpha_f f(y)\\
& \geq &\alpha_f \cdot f(S \cup y)\\
& \geq & \alpha_f \cdot (1-e^{-\alpha_f+o(1)})  \cdot f(OPT).
\end{eqnarray*}

We have $f(T)= \max\{f(S), f(y)\}$ and hence
$$f(T) \geq (\alpha_f/2) (1-e^{-\alpha_f+o(1)})  \cdot f(OPT).
$$

 As $f$ is monotone and each $P_i$ is obtained from $T$ by adding an additional set of elements, we have for
$$f(P_i) \geq f(T) \geq (\alpha_f/2) (1-e^{-\alpha_f+o(1)}) f(OPT)$$
$1 \leq i \leq \mu$, which completes the proof.
\end{proof}

We now investigate the number of different solutions meeting the quality threshold determined by $B-m$.

\begin{theorem}
\label{thm:gdgs-size}
There are at least 
$$\sum_{i=1}^{\lfloor m/c_{\max}\rfloor} \binom{n-\lfloor (B-m)/c_{\min} \rfloor}{i}$$
distinct feasible solutions $X$ with $f(X) \geq (\alpha_f/2) (1-e^{-\alpha_f+o(1)})\cdot f(OPT)$.
\end{theorem}
\begin{proof}
Obtaining the set $T$ of costs at most $B-m$, we can pick at least $\lfloor m/c_{\max} \rfloor$ additional elements. For the number of elements in the set $T$ produced by GDGS, we have $|T|\leq \lfloor (B-m)/c_{\min} \rfloor$. This implies that there are
at least
$$
\sum_{i=1}^{\lfloor m/c_{\max}\rfloor} \binom{n-\lfloor (B-m)/c_{\min} \rfloor}{i} 
$$
different sets of solutions with the desired approximation quality.
\end{proof}
Note, that in order to obtain $\mu$ of these solutions, GDGS creates $\mu$ solutions by randomly selecting unchosen elements until no further element can be included without violating the constraint bound $B$. This implies that in most situations, solutions with a cost close to the constraint bound $B$ are generated.

\section{Entropy-Based Evolutionary Diversity Optimisation}

We now introduce a simple evolutionary diversity optimisation approach to create high quality diverse sets of solutions.

In order to maximise diversity of our set of $\mu$ solutions we introduce a threshold $f_T$ and aim to produce a set of solutions maximizing diversity under the condition that for all $x \in P$, $f(x) \geq f_{\min}$ holds.

\subsection{Entropy-Based Diversity Measure}
We define the entropy of a given population $P$ for the set of input elements $V=\{v_1, \ldots, v_n\}$ as
$$
H(P) = - \sum_{i=1}^n p(v_i) \log_2 p(v_i).
$$
and aim to maximize this value under the condition that all solutions in $P$ meet a given quality criterion.
We use $\log(x)$ instead of $\log_2(x)$ in the following to simplify notations.
For a given population $P$ of size $\mu$, we use
$$p(v_i)= |\{P_j \in P\mid v_i \in P_j\}| / \mu$$
which equals the fraction of solutions in $P$ that contain element $v_i$.
Note, that elements not present in the population or elements present in all solutions have a contribution of $0$ to $H(P)$.
The function $-p(x) \log(x)$ is concave in $[0,1]$ and monotonically increasing in $[0,1/e]$.
This implies that increasing the fraction of solutions in $P$ until a fraction of $1/e$ is obtained is rewarded by the entropy diversity measure. In addition $-p(x) \log(x)$ is monotonically decreasing in $[1/e,1]$ which means that increasing fractions to larger than $1/e$ reduces diversity with respect to the considered entropy measure.

\subsection{Limitations of Diversifying Greedy Sampling}
Before introducing the evolutionary diversity optimisation approach, we point out some limitations of diversifying greedy sampling that show in which way diversity may be improved in common situations.
We consider DGS which selects for each solution $P_i$ the same set of elements $S$ with $|S|=B-m$ and adds exactly $m$ additional elements afterwards. This implies that there is no contribution of the elements in $S$ to the entropy score and we can upper bound the entropy of a population as follows.

\begin{theorem}
Let $P$ be a population produced by DGS. Then we have 
$$H(P) \leq -m \log \left(\frac{m}{n-B+m} \right) \leq \frac{n-B+m}{e} \cdot \log (1/e).$$
\end{theorem}
\begin{proof}
The entropy of a population is maximal if the number of occurrences of all elements in $P$ differs by at most $1$ and the fraction of occurrences of all elements is $1/e$.
This is due to the fact that the function $f(x) = - x \log x$ is monotonically increasing in $[0,1/e]$ and monotonically decreasing in $[1/e,1]$ and that the second derivative $f''(x) = -\frac{1}{\ln(2)\cdot x}$ is negative in $[0,1]$.
We assume that the first $B-m$ elements are contained in each of the $\mu$ solutions. This means that these elements do not contribute to the entropy measure
The maximum entropy of a population $P$ obtained by DGS is

\begin{eqnarray*}
H(P) & = & - \sum_{i=1}^n p(v_i) \log p(v_i) \\
& = & -\sum_{i=B-m+1}^{n} p(v_i) \log p(v_i)\\
& \leq & -\sum_{i=B-m+1}^{n} \frac{m}{n-B+m} \log \left(\frac{m}{n-B+m} \right)\\
& = & -(n-B+m) \frac{m}{n-B+m} \log \left(\frac{m}{n-B+m} \right)\\
& = & -m \cdot \log \left(\frac{m}{n-B+m} \right)\\
& \leq & (n-B+m) \cdot \frac{1}{e} \cdot \log (1/e) 
\end{eqnarray*}

\end{proof}

To illustrate the benefit of using the $B-m$ elements of $S$ to improve diversity, we consider the simple example of the function $OneMax(X) = |X|$ which is the simplest non-trivial submodular function.

Assume that $B/n$ is an integer. For OneMax, distributing $\mu B$ elements equally among the $n$ positions would give a population $P^*$ with $p(v_i) = B/n, 1 \leq i \leq n$ and we have 

\begin{eqnarray*}
H(P^*) & = & - \sum_{i=1}^n p(v_i) \log p(v_i) \\
& = & -\sum_{i=1}^n (B/n) \log (B/n)\\
& = & -B\log (B/n)
\end{eqnarray*}

%Note that we neglected rounding to integers here to ease the presentation.

Assuming $B=\lceil n/e \rceil$ (or $B=\lfloor n/e \rfloor$ depending on rounding to the next integer), then $H(P^*)$ is the maximal among all possible populations if each item appears in the population $\lfloor \mu/e \rfloor$ or $\lceil \mu/e \rceil$ times. In this case, the maximal entropy value that can be obtained is approximately $H(P) = \frac{n}{e} \cdot \log(1/e)$.

\subsection{Evolutionary Diversity Optimisation}

The aim is to further improve the sets of solutions created by the diversifying greedy sampling approaches. Given the set of solutions $P=\{P_1,  \ldots, P_{\mu}\}$ produced by a diversifying greedy approach, we set a quality threshold $f_{\min}= \min_{i=1}^{\mu} f(P_i)$ and improve diversity of the set $P$ under the condition that all solutions in $P$ are feasible and have function value at least $f_{\min}$. 

We use Algorithm~\ref{alg:divea} to compute a high quality diverse population where each individual $I$ has to meet a given quality criteria $f_{\min}$ according to a given budget. An individual $I \subseteq V$ is a set of elements of $V$. 
At first, the initial population $P$ is generated with $\mu$ individuals created by the diversifying greedy sampling approaches. In each iteration of DIVEA exactly one offspring $I'$ is produced by mutation.
We use a mutation operator matching standard bit mutations for bit-strings. Given an individual $I$ we produce a new individual $I'$ by copying $I$ and changing the status of each $v_i \in V$ for being included or excluded in $I'$ with probably $1/n$. If the offspring $I'$ meets the quality threshold $f_{\min}$ and the budget constraint, then $I'$ is added to the population. If $I'$ is added to the population, one individual is selected for removal ensuring that the population size after each iteration is $\mu$. An individual $I \in P$ is removed such that $H(P\setminus \{I\})$ is maximal among all individuals $J \in P$. This implies that the resulting population $P$ has the highest possible entropy of all sets of $\mu$ individuals available in iteration $t$.
The algorithm iterates for $t_{\max}$ iterations and outputs the final population.

It should be noted that DIVEA is a very simple baseline algorithm and the goal here is to show how evolutionary algorithms can be used to further improve diversity of the sets obtained by the diversifying greedy sampling approaches.

\begin{algorithm}[t]
\SetKwInOut{Input}{input}
\Input{
Initial population $P=\{P_1, \ldots, P_{\mu}\}$, threshold $f_{\min}= \min_{i=1}^{\mu} f(P_i)$, maximum number of iterations $t_{\max}$.}
$S \leftarrow P_i$\;
$t \leftarrow 0$\;
\While{$t\leq t_{\max}$}{
$t \leftarrow t+1$\;
Choose $I \in P$ uniformly at random and produce an offspring $I'$ of $I$ by mutation\;
\If{$(f(I') \geq f_{\min}) \wedge (c(I) \leq B)$}{ $P \leftarrow P \cup \{I'\}$\;
Remove exactly one individual $I$, with $I=\arg \max_{J \in P} H(P\setminus \{J\})$ from $P$\;}
}
 \Return{$P=\{P_1, \ldots, P_{\mu}\}$}\;
\caption{Diversifying EA (DIVEA)}
\label{alg:divea}
\end{algorithm}

\section{Experimental Investigations}
We examine the introduced algorithms on the submodular influence maximization problem~\citep{DBLP:conf/ijcai/QianSYT17,DBLP:conf/kdd/LeskovecKGFVG07} and the maximum coverage problem~\cite{DBLP:journals/ipl/KhullerMN99,DBLP:journals/jacm/Feige98}.

\subsection{The Influence Maximization Problem}
The influence maximization problem (IM) aims to identify a set of the most influential users in a social network. IM intents to maximise the spread of influence through a social network i.e. a graph of social interactions within a group of users~\cite{DBLP:conf/kdd/KempeKT03}.

The social network is modeled as a directed graph $G=(V,E)$ where each node represents a user, and each edge $(u,v) \in E$ has been assigned an edge probability $p_{u,v}$ which indicates that user $u$ influences user $v$. The aim of the influence maximization problem is to find a subset $X \subseteq V$ such that the number of activated nodes of $X$ is maximised.

\subsection{The Maximum Coverage Problem}
The maximum coverage problem~\cite{DBLP:journals/ipl/KhullerMN99,DBLP:journals/jacm/Feige98} is a classical NP-hard submodular optimisation problem and arises frequently in a variety of settings. Given a set $U$ of elements, a collection $V = \{V_1,V_2,\ldots,V_n\}$ of subsets of $U$, a cost function $c: 2^V\rightarrow \R^+$
and a budget $B$, the goal is to find a collection  of subsets $X^* \subseteq V$ such that the number of covered elements is maximized subject to meeting the cost constraint, i.e.
$$X^* = \argmax_{X  \subseteq  V}  \left\{|\cup_{V_i \in X} V_i| \mid c(X) \leq B \right\}.$$

\subsection{Experimental setting}
For our experiments, we use real-world graph frb30-15-01 that contains $450$ nodes and $17\,827$ edges from~\cite{datasetsfrb}. In the case of the maximum coverage problem, the set $U$ consist of the vertices of the graph and for each vertex $v_i$, we construct a set $V_i$ that includes the vertex itself and its adjacent vertices with a higher node number.

We consider uniform and knapsack constraints.
In the uniform setting the cost of a set of chosen nodes $X$ for the maximum influence problem is $c(X) = |X|$. In the uniform setting for maximum coverage problem, the cost is a solution $X$ is given by the number of chosen sets, i.e. we have $c(X) = |\{V_i \mid V_i \in X\}|$.
In the case of knapsack constraints, the cost of a node $v$ in the influence maximization problem is given by $c(v) = \text{deg}(v)+1$, where $\text{deg}(v)$ denotes the outdegree of $v$ in $G$. The cost of a given set of nodes $X$ is given as $c(X) = \sum_{v \in X} c(v)$.
For the maximum coverage problem, the cost of a set $V_i$ is $c(V_i) = |V_i|$ and the cost of a solution $X$ is given as $c(X) = \sum_{V_i \in X}|V_i|$.
%%%%%%%%%%%%%%%%%%%%%%%%%%%%%%%%%%%%%%%%%%%%%%%%%%%%%%%%%%%%%%%%%%
%%%%%%%%%%%%%%%%%%%%%%%%%%%%%%%%%%%%%%%%%%%%%%%%%%%%%%%%%%%%%%%%%%
\begin{table}[!t]
\renewcommand{\arraystretch}{1.4}  
\renewcommand\tabcolsep{4pt} 
\begin{tiny}
\centering
\begin{tabular}{@{}ccccc|rrc|rrc}
\toprule  
     
\multirow{2}{*}{$B$} & \multirow{2}{*}{$m$} &
\multirow{2}{*}{$\mu$} & \multicolumn{2}{c}{\bfseries Threshold (1)} & \multicolumn{3}{c}{\bfseries DGS Entropy (2)} & \multicolumn{3}{c}{\bfseries DIVEA Entropy (3)}\\
\cmidrule(l{2pt}r{2pt}){4-5} \cmidrule(l{2pt}r{2pt}){6-8} \cmidrule(l{2pt}r{2pt}){9-11} 
 &  &  & \textbf{mean}& \textbf{std}  & \textbf{mean} & \textbf{std} & \textbf{stat} &\textbf{mean} & \textbf{std}& \textbf{stat} \\
\midrule
\multirow{2}{*}{10}&2&5&43.13&0.1826&4.6039&0.1221&$3^{(-})$&\textbf{23.2193}&0.0000&$2^{(+)}$\\   
                    &2&10&40.07&1.6758 & 6.5572&0.1456&$3^{(-)}$&\textbf{33.1793}&0.2191&$2^{(+)}$\\

\multirow{2}{*}{10}&2&15&39.46&1.2249&7.7160&0.1353&$3^{(-)}$&\textbf{39.0422}&0.1074&$2^{(+)}$\\    
                    &2&20&39.39&1.3378&8.4572&0.1332&$3^{(-)}$&\textbf{43.1293}&0.1826&$2^{(+)}$\\

\midrule
\multirow{2}{*}{10}&5&5&40.37&1.8864& 11.4096&0.2729&$3^{(-)}$&\textbf{23.2193}&0.0000&$2^{(+)}$\\    
                      &5&10&39.22&1.4137&16.2096&0.2828&$3^{(-)}$&\textbf{33.2193}&0.0000&$2^{(+)}$\\ 

\multirow{2}{*}{10}&5&15&38.80&1.2474&  18.7255&0.3123&$3^{(-)}$&\textbf{39.0645}&0.0243&$2^{(+)}$\\    
                 &5&20&38.76&1.0637&  20.6854&0.2107&$3^{(-)}$&\textbf{43.1780}&0.0948&$2^{(+)}$\\

\midrule
\multirow{2}{*}{10}&8&5&37.81&1.5747 &18.0954&0.5594&$3^{(-)}$&\textbf{23.2193}&0.0000&$2^{(+)}$\\    
                   &8&10&37.52&1.6507&25.2803&0.6395&$3^{(-)}$&\textbf{33.2193}&0.0000&$2^{(+)}$\\

\multirow{2}{*}{10}&8&15&37.47&1.5072&29.4682&0.4243&$3^{(-)}$&\textbf{39.0689}&0.0000&$2^{(+)}$\\    
                    &8&20&37.27&0.8154&32.2052&0.3433&$3^{(-)}$&\textbf{43.2193}&0.0000&$2^{(+)}$\\
\midrule
\end{tabular}
\caption{Results for the influence maximization problem with uniform constraints for $B = 10$.}

\label{tb:Results_1}
\end{tiny}
\end{table}    
%%%%%%%%%%%%%%%%%%%%%%%%%%%%%%%%%%%%%%%%%%%%%%%%%%%%%%%%%%%%%%%%%%
%%%%%%%%%%%%%%%%%%%%%%%%%%%%%%%%%%%%%%%%%%%%%%%%%%%%%%%%%%%%%%%%%%

\begin{table}[!t]
\renewcommand{\arraystretch}{1.4}  
\renewcommand\tabcolsep{3pt} 
\centering
\begin{tiny}
\begin{tabular}{@{}ccccc|rrc|rrc}
\toprule 
     
\multirow{2}{*}{$B$} & \multirow{2}{*}{$m$} &
\multirow{2}{*}{$\mu$} & \multicolumn{2}{c}{\bfseries Threshold (1)} & \multicolumn{3}{c}{\bfseries DGS Entropy (2)} & \multicolumn{3}{c}{\bfseries DIVEA Entropy (3)}\\
\cmidrule(l{2pt}r{2pt}){4-5} \cmidrule(l{2pt}r{2pt}){6-8} \cmidrule(l{2pt}r{2pt}){9-11} 
 &  &  & \textbf{mean}& \textbf{std}  & \textbf{mean} & \textbf{std} & \textbf{stat} &\textbf{mean} & \textbf{std}& \textbf{stat} \\
\midrule

\multirow{2}{*}{10}&2&5&429.70&1.2360&4.6439&0.0000&$3^{(-})$&\textbf{18.0233} &1.9351&$2^{(+)}$\\    
&2&10&428.60&0.9322&6.5972&0.1137&$^{3(-)}$&\textbf{23.2874}&1.4800&$2^{(+)}$\\

\multirow{2}{*}{10}&2&15&427.90&0.9948&7.6921&0.1320&$3^{(-)}$&\textbf{25.4377} &1.5628&$2^{(+)}$\\    
&2&20&427.50&1.0009&8.4959&0.9010&$3^{(-)}$&\textbf{26.6811}&1.8844&$2^{(+)}$\\

\midrule
\multirow{2}{*}{10}&5&5&429.53&1.4320&11.3263&0.2959&$3^{(-)}$&\textbf{23.2059}&0.0731&$2^{(+)}$\\    
&5&10&428.83&1.0530&16.0237&0.3192&$3^{(-)}$&\textbf{33.1276}&0.1868&$2^{(+)}$\\

\multirow{2}{*}{10}&5&15&427.80&1.0635&18.7877&0.3068&$3^{(-)}$&\textbf{38.6461}&0.3739&$2^{(+)}$\\    
&5&20&427.73&0.8683&20.5862&0.2211&$3^{(-)}$&\textbf{42.0368}&0.5832&$2^{(+)}$\\
\midrule
\multirow{2}{*}{10}&8&5&383.83&8.2716&18.1037&0.4311&$3^{(-)}$&\textbf{23.2193}&0.0000&$2^{(+)}$\\    
&8&10&382.77&6.5003&25.2569&0.5274&$3^{(-)}$&\textbf{33.2198}&0.0000&$2^{(+)}$\\ 

\multirow{2}{*}{10}&8&15&378.97&5.4818&29.3765&0.5114&$3^{(-)}$&\textbf{39.0689}&0.0000&$2^{(+)}$\\   
&8&20&374.90&7.4941&32.1019&0.4154&$3^{(-)}$&\textbf{43.2193}&0.0000&$2^{(+)}$\\
\midrule
%%%%%%%%%%%%%%%%%%%%%%%%%%%%%%%%%%%%%%%%%%%%%%%%%%%%%%%%%%%%%%%%
%%%%%%%%%%%%%%%%%%%%%%%%%%%%%%%%%%%%%%%%%%%%%%%%%%%%%%%%%%%%%%%%

\multirow{2}{*}{15}&2&5&449.00&0.0000&4.6305&0.0731&$3^{(-})$&\textbf{27.6151}&2.2429&$2^{(+)}$\\    
&2&10&449.00&0.0000&6.5639&0.5704&$3^{(-)}$&\textbf{31.0779}&2.1572&$2^{(+)}$\\

\multirow{2}{*}{15}&2&15&449.00&0.0000&7.6565&0.1479&$3^{(-)}$&\textbf{33.0764}&1.7469&$2^{(+)}$\\    
&2&20&449.00&0.0000&8.5205&0.0917&$3^{(-)}$&\textbf{33.7519}&1.8101&$2^{(+)}$\\
\midrule
\multirow{2}{*}{15}&5&5&444.17&1.1167&11.3563&0.3401&$3^{(-)}$&\textbf{34.6772}&0.3627&$2^{(+)}$\\    
&5&10&443.80&1.0954&16.0621&0.2732&$3^{(-)}$&\textbf{47.8240}&1.3957&$2^{(+)}$\\

\multirow{2}{*}{15}&5&15&443.47&1.0742&18.6366&0.2732&$3^{(-)}$&\textbf{53.7846}&1.9886&$2^{(+)}$\\    
&5&20&442.90&0.9595&20.5957&0.2840&$3^{(-)}$&\textbf{57.4155}&1.9886&$2^{(+)}$\\
\midrule
\multirow{2}{*}{15}&8&5&435.77&1.9061&18.0154&0.4014&$3^{(-)}$&\textbf{34.8289}&0.0000&$2^{(+)}$\\   
&8&10&434.37&2.2047&25.3553&0.4093&$3^{(-)}$&\textbf{49.8223}&0.0365&$2^{(+)}$\\

\multirow{2}{*}{15}&8&15&434.23&1.9420&29.2837&0.4798&$3^{(-)}$&\textbf{58.4817}&0.1664&$2^{(+)}$\\    
&8&20&434.17&1.4162&32.0748&0.4107&$3^{(-)}$&\textbf{64.0082}&0.3968&$2^{(+)}$\\
\midrule

\end{tabular}
\caption{Results for the maximum coverage problem with uniform constraints for $B = 10, 15$.}
\label{tb:Results_2}
\end{tiny}
\end{table}    
%%%%%%%%%%%%%%%%%%%%%%%%%%%%%%%%%%%%%%%%%%%%%%%%%%%%%%%%%%%%%%%%%%
\subsection{Diverse Sets for Uniform Constraints}
We consider the results for the diversifying greedy sampling and diversifying evolutionary algorithm with all the weights $1$.

The experimental results of the influence maximization problem and maximum coverage problem for the DGS and DIVEA are shown in Table~\ref{tb:Results_1} and Table~\ref{tb:Results_2}, respectively. For the experimental investigations, we consider all combinations of $m = 2, 5, 8$ and $\mu = 5, 10, 15, 20$ for $B = 10, 15$. For each instance, we run each algorithm $30$ times and record the final population. We round the standard deviation to $4$ decimal points. Note that for estimating the influence spread, we simulate the information diffusion process among the users $100$ times independently.  

We compare the results in terms of the entropy values obtained by the DGS and DIVEA at each $m$ and $\mu$ for budgets $B = 10, 15$. In order to test the statistical significance of the results we use the Kruskal-Wallis test with $95\%$ confidence in order to measure the statistical validity of our results. We apply the Bonferroni post-hoc statistical procedure that is used for multiple comparisons of a control algorithm to two or more algorithms~\cite{Corder09}. $Y^{(+)}$ is equivalent to the statement that the algorithm in the column outperformed algorithm $Y$ (see numbers behind algorithm names in the top rows of the tables). $Y^{(-)}$ is equivalent to the statement that $Y$ outperformed the algorithm given in the column.

The results in the Table~\ref{tb:Results_1} show that the DIVEA for the influence maximization problem performs significantly better with respect to diversity than the DGS approach for $B = 10$. We see a similar result in terms of threshold values where DIVEA is able to attain the quality level achieved by DGA and additionally to produce significantly higher entropy values than DGS for all of the cases. In particular, the DIVEA creates higher entropy values in comparison to the results produced by DSG for the smallest margin $m = 2$.

Table~\ref{tb:Results_2} shows that DIVEA obtains higher entropy values than the DGS among all considered combinations of $m$ and $\mu$. It should be noted that higher entropy values $H(P)$ indicate a higher diversification of the population. In the case of a higher margin $m$, e.g. $m =8$, DIVEA performs consistently better. Furthermore, Table~\ref{tb:Results_2} includes threshold values obtained by the DGS for all the combinations of $m$ and $\mu$. We observe that the threshold values decrease with increasing size of the margin and number of individuals in the population. The results suggest that the DGS algorithm is able to create diverse solutions and simultaneously maintain a similar quality. Predominantly, DIVEA is able to provide the approximation quality achieved by the greedy approaches and to achieve significantly higher entropy values than DGS for all of the cases.

\subsection{Diverse Sets for Knapsack Constraints}

We now consider the generalized diversifying greedy sampling and diversifying evolutionary algorithm for the setting where the weights depend on the degree of the node of the graph. We consider budgets $B = 100$ and margins $m = 10, 20, 30$. The sizes of the populations is the same as for the uniform constraints. Table~\ref{tb:Results_3} and Table~\ref{tb:Results_4} include threshold and entropy values obtained by the GDGS and the DIVEA for the combination of $m$ and $\mu$ for the influence maximization problem and maximum coverage problem. For the setting, where $B$ elements can be used, the obtained threshold value for the GDGS and $B = 100$, $\mu =5$ and $m =10, 20, 30$ is $48.34$, $46.04$ and $45.88$ for the influence maximization problem, respectively. Additionally, the threshold value for the GDGS and $B = 100$, $\mu =5$ and $m =10, 20, 30$ is 
$406.30$, $398.53$ and $388.57$ for the maximum coverage problem.
Furthermore, we compare the results in terms of the entropy values obtained by the GDGS and the DIVEA. We observe that the DIVEA outperforms the GDGS for all combinations of $B$, $m$ and $\mu$ for both problems. The entropy values of the approaches are overall increasing for each margin value when the number of populations increases. The results also show that the DIVEA is able to more directly improve diversity of the population gathered by the GDGS approach as the number of the margin increases. 
%%%%%%%%%%%%%%%%%%%%%%%%%%%%%%%%%%%%%%%%%%%%%%%%%%%%%%%%%%%%%%%%%%
\begin{table}[!t]
\renewcommand{\arraystretch}{1.4}  
\renewcommand\tabcolsep{4pt} 
\centering
\begin{tiny}
\begin{tabular}{@{}ccccc|rrc|rrc}
\toprule  
     
\multirow{2}{*}{$B$} & \multirow{2}{*}{$m$} &
\multirow{2}{*}{$\mu$} & \multicolumn{2}{c}{\bfseries Threshold (1)} & \multicolumn{3}{c}{\bfseries DGS Entropy (2)} & \multicolumn{3}{c}{\bfseries DIVEA Entropy (3)}\\
\cmidrule(l{2pt}r{2pt}){4-5} \cmidrule(l{2pt}r{2pt}){6-8} \cmidrule(l{2pt}r{2pt}){9-11} 
 &  &  & \textbf{mean}& \textbf{std}  & \textbf{mean} & \textbf{std} & \textbf{stat} &\textbf{mean} & \textbf{std}& \textbf{stat} \\
\midrule
\multirow{2}{*}{100}&10&5&48.34&3.0483&2.3422&0.4522&$3^{(-})$&\textbf{14.7342}&2.4239&$2^{(+)}$\\    
                    &10&10&47.57&3.2168 & 2.9609&0.4285&$3^{(-)}$&\textbf{16.9616}&2.5704&$2^{(+)}$\\

\multirow{2}{*}{100}&10&15&47.54&2.5498&3.4633&0.5369&$3^{(-)}$&\textbf{17.1217}&2.0946&$2^{(+)}$\\    
                    &10&20&47.50&2.7268&3.4882&0.4566&$3^{(-)}$&\textbf{17.7975}&2.2266&$2^{(+)}$\\

\midrule
\multirow{2}{*}{100}&20&5&46.04&5.0085& 3.2619&0.4084&$3^{(-)}$&\textbf{15.7603}&2.4532&$2^{(+)}$\\    
                      &20&10&45.70&3.3142&4.6998&0.4005&$3^{(-)}$&\textbf{18.5986}&2.4380&$2^{(+)}$\\

\multirow{2}{*}{100}&20&15&45.61&2.1911&  5.4380&0.5318&$3^{(-)}$&\textbf{19.4041}&1.4929&$2^{(+)}$\\    
                 &20&20&45.45&2.5654&  5.7699&0.4936&$3^{(-)}$&\textbf{20.0244}&1.5025&$2^{(+)}$\\

\midrule
\multirow{2}{*}{100}&30&5&45.88&1.6549 &4.1523&0.5008&$3^{(-)}$&\textbf{16.1894}&2.1960&$2^{(+)}$\\    
                   &30&10&44.03&3.9920&6.1409&0.6590&$3^{(-)}$&\textbf{19.4785}&1.8216&$2^{(+)}$\\

\multirow{2}{*}{100}&30&15&41.62&3.1492&7.0700&0.3484&$3^{(-)}$&\textbf{21.4490}&2.0305&$2^{(+)}$\\    
                    &30&20&40.17&1.7436&7.1638&0.3493&$3^{(-)}$&\textbf{22.1779}&1.1061&$2^{(+)}$\\
\midrule
\end{tabular}
\caption{Results for the influence maximization problem with knapsack constraints for $B = 100$.}

\label{tb:Results_3}
\end{tiny}
\end{table}    
%%%%%%%%%%%%%%%%%%%%%%%%%%%%%%%%%%%%%%%%%%%%%%%%%%%%%%%%%%%%%%%%%%
%%%%%%%%%%%%%%%%%%%%%%%%%%%%%%%%%%%%%%%%%%%%%%%%%%%%%%%%%%%%%%%%%%
\begin{table}[!t]
\renewcommand{\arraystretch}{1.4}  
\renewcommand\tabcolsep{4pt} 
\centering

\begin{tiny}
\begin{tabular}{@{}ccccc|rrc|rrc}
\toprule  
     
\multirow{2}{*}{$B$} & \multirow{2}{*}{$m$} &
\multirow{2}{*}{$\mu$} & \multicolumn{2}{c}{\bfseries Threshold (1)} & \multicolumn{3}{c}{\bfseries DGS Entropy (2)} & \multicolumn{3}{c}{\bfseries DIVEA Entropy (3)}\\
\cmidrule(l{2pt}r{2pt}){4-5} \cmidrule(l{2pt}r{2pt}){6-8} \cmidrule(l{2pt}r{2pt}){9-11} 
 &  &  & \textbf{mean}& \textbf{std}  & \textbf{mean} & \textbf{std} & \textbf{stat} &\textbf{mean} & \textbf{std}& \textbf{stat} \\
\midrule
\multirow{2}{*}{100}&10&5&406.30&0.7022&2.3190&0.3801&$3^{(-})$&\textbf{5.1566}&1.2727&$2^{(+)}$\\    
                    &10&10&406.03&0.1826 & 3.1748&0.2398&$3^{(-)}$&\textbf{5.7382}&0.9848&$2^{(+)}$\\

\multirow{2}{*}{100}&10&15&406.00&0.0000&3.4411&0.2587&$3^{(-)}$&\textbf{6.1239}&0.9848&$2^{(+)}$\\    
                    &10&20&406.00&0.0000&3.6965&0.2307&$3^{(-)}$&\textbf{6.6749}&1.0283&$2^{(+)}$\\

\midrule
\multirow{2}{*}{100}&20&5&398.53&1.6761& 3.5986&0.4343&$3^{(-)}$&\textbf{10.0783}&1.4947&$2^{(+)}$\\    
                      &20&10&397.43&1.0726&5.0012&0.3948&$3^{(-)}$&\textbf{11.5237}&1.4219&$2^{(+)}$\\

\multirow{2}{*}{100}&20&15&397.07&0.9444&  5.7040&0.3835&$3^{(-)}$&\textbf{11.8965}&0.9952&$2^{(+)}$\\    
                 &20&20&396.77&0.9714&  6.1189&0.4051&$3^{(-)}$&\textbf{12.8706}&1.2013&$2^{(+)}$\\

\midrule
\multirow{2}{*}{100}&30&5&388.57&2.3294 &4.3088&0.6259&$3^{(-)}$&\textbf{13.4104}&1.4536&$2^{(+)}$\\    
                   &30&10&387.93&2.1804&6.2473&0.5190&$3^{(-)}$&\textbf{14.7949}&1.3123&$2^{(+)}$\\

\multirow{2}{*}{100}&30&15&386.97&2.1573&7.2866&0.4094&$3^{(-)}$&\textbf{15.7160}&1.0320&$2^{(+)}$\\    
                    &30&20&386.17&1.7436&7.8370&0.4966&$3^{(-)}$&\textbf{16.1779}&1.1061&$2^{(+)}$\\
\midrule
\end{tabular}
\caption{Results for the maximum coverage problem with knapsack constraints for $B = 100$.}

\label{tb:Results_4}
\end{tiny}
\end{table}    
%%%%%%%%%%%%%%%%%%%%%%%%%%%%%%%%%%%%%%%%%%%%%%%%%%%%%%%%%%%%%%%%%%
%%%%%%%%%%%%%%%%%%%%%%%%%%%%%%%%%%%%%%%%%%%%%%%%%%%%%%%%%%%%%%%%%%
\subsection{Conclusions}
We have presented approaches for creating diverse sets of solutions for monotone functions under given constraints. Our diversifying greedy sampling approaches create sets of solutions with provable guarantees that match the current best performance ratios obtained by greedy algorithms. Furthermore, we have examined the short-comings in terms of the entropy diversity measure and proposed an entropy-based evolutionary diversity optimisation approach to improve the diversity of the populations obtained by the diversifying greedy sampling approaches. Our experimental results show that high quality sets of solutions can be obtained for important submodular optimisation problems and that the evolutionary diversity optimisation approach significantly increases the entropy diversity of the sets created.
%%%%%%%%%%%%%%%%%%%%%%%%%%%%%%%%%%%%%%%%%%%%%%%%%%%
\section{Acknowledgements}
This work has been supported by the Australian Research Council through grant DP190103894, and by the South Australian Government through the Research Consortium "Unlocking Complex Resources through Lean Processing".

\bibliography{references}

\end{document}